\documentclass[lettersize,journal]{IEEEtran}
\usepackage{amsmath,amsfonts}
\usepackage{algorithmic}
\usepackage{algorithm}
\usepackage{array}
\ifCLASSOPTIONcompsoc
\usepackage[caption=false, font=normalsize, labelfont=sf, textfont=sf]{subfig}
\else
\usepackage[caption=false, font=footnotesize]{subfig}
\fi
\usepackage{textcomp}
\usepackage{stfloats}
\usepackage{url}
\usepackage{verbatim}
\usepackage{graphicx}
\usepackage{cite}
\usepackage{multirow}
\usepackage{colortbl}
\usepackage[table ]{ xcolor}
\usepackage{booktabs}
\usepackage{bm}
\begin{document}

\def \x {\bm{x}}
\def \y {\bm{y}}
\def \h {\bm{h}}
\def \w {\bm{w}}

\title{Meta Objective Guided Disambiguation for Partial Label Learning}

\author{Bo-Shi Zou, Ming-Kun Xie, and Sheng-Jun Huang
\thanks{Preprint. Under review.}}

\markboth{Journal of \LaTeX\ Class Files, August 2022}%
{Shell \MakeLowercase{\textit{et al.}}: A Sample Article Using IEEEtran.cls for IEEE Journals}



\maketitle

\begin{abstract}
Partial label learning (PLL) is a typical weakly supervised learning framework, where each training instance is associated with a candidate label set, among which only one label is valid. To solve PLL problems, typically methods try to perform disambiguation for candidate sets by either using prior knowledge, such as structure information of training data, or refining model outputs in a self-training manner. Unfortunately, these methods often fail to obtain a favorable performance due to the lack of prior information or unreliable predictions in the early stage of model training. In this paper, we propose a novel framework for partial label learning with meta objective guided disambiguation (MoGD), which aims to recover the ground-truth label from candidate labels set by solving a meta objective on a small validation set. Specifically, to alleviate the negative impact of false positive labels, we re-weight each candidate label based on the meta loss on the validation set. Then, the classifier is trained by minimizing the weighted cross entropy loss. The proposed method can be easily implemented by using various deep networks with the ordinary SGD optimizer. Theoretically, we prove the convergence property of meta objective and derive the estimation error bounds of the proposed method. Extensive experiments on various benchmark datasets and real-world PLL datasets demonstrate that the proposed method can achieve competent performance when compared with the state-of-the-art methods.

\end{abstract}

\begin{IEEEkeywords}
partial label learning, candidate label set, ground-truth label, disambiguation, meta-learning.
\end{IEEEkeywords}

\section{Introduction}
\label{sec: Introduction}
\IEEEPARstart{W}{ith} the increasing amounts of carefully labeled data used for training modern machine learning models,
model performance has been greatly improved. Nevertheless, collecting a large number of precisely labeled training data is time-consuming and costly in many realistic applications, which imposes a trade-off between the classification performance and the labeling cost.

Partial label learning (PLL) is a commonly used weakly supervised learning framework, where each training example is associated with a set of candidate labels, among which only one corresponds to the ground-truth label\cite{cour2011learning,jin2002learning}. For example, as illustrate in Fig. \ref{fig:msn}, in face image recognition (see Fig. \ref{fig:msn})\cite{guillaumin2008automatic,zeng2013learning,chen2014ambiguously}, an image with multiple faces is often associated with textual description. One can treat each face detected from the image as an example and those names extracted from the associated textual description consist of the candidate label set of each example. In crowdsourcing image tagging (see Fig. \ref{fig:husky}), an image would be assigned with different labels by labelers with different level of expertise. PLL has been successfully applied into numerous realistic applications, such as web mining \cite{feng2019partial,luo2010learning}, Ecoinformatics\cite{zhou2016partial}, multimedia content analysis\cite{chen2015matrix,dietterich1994solving}, etc.

\begin{figure}[!t]
	\centering
	\subfloat[Automatic face naming]
	{
		\includegraphics[width=0.21\textwidth]{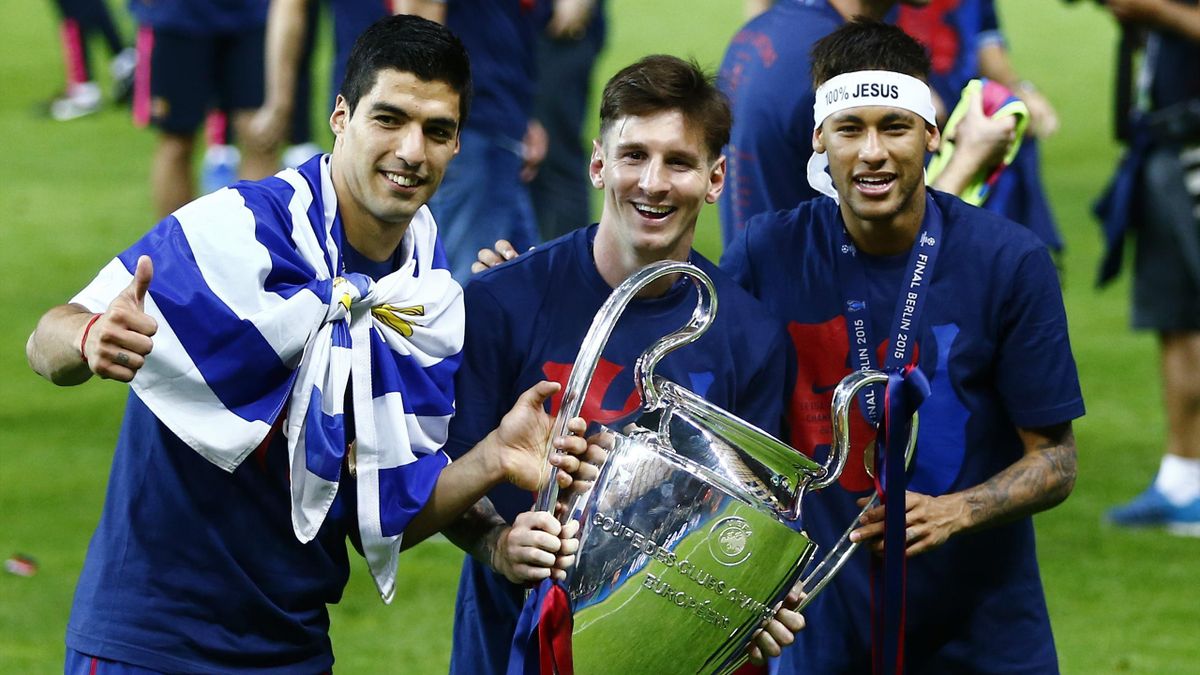}
		\label{fig:msn}
	}
	\hfil
	\subfloat[Crowdsourcing tagging]
	{
		\includegraphics[width=0.21\textwidth]{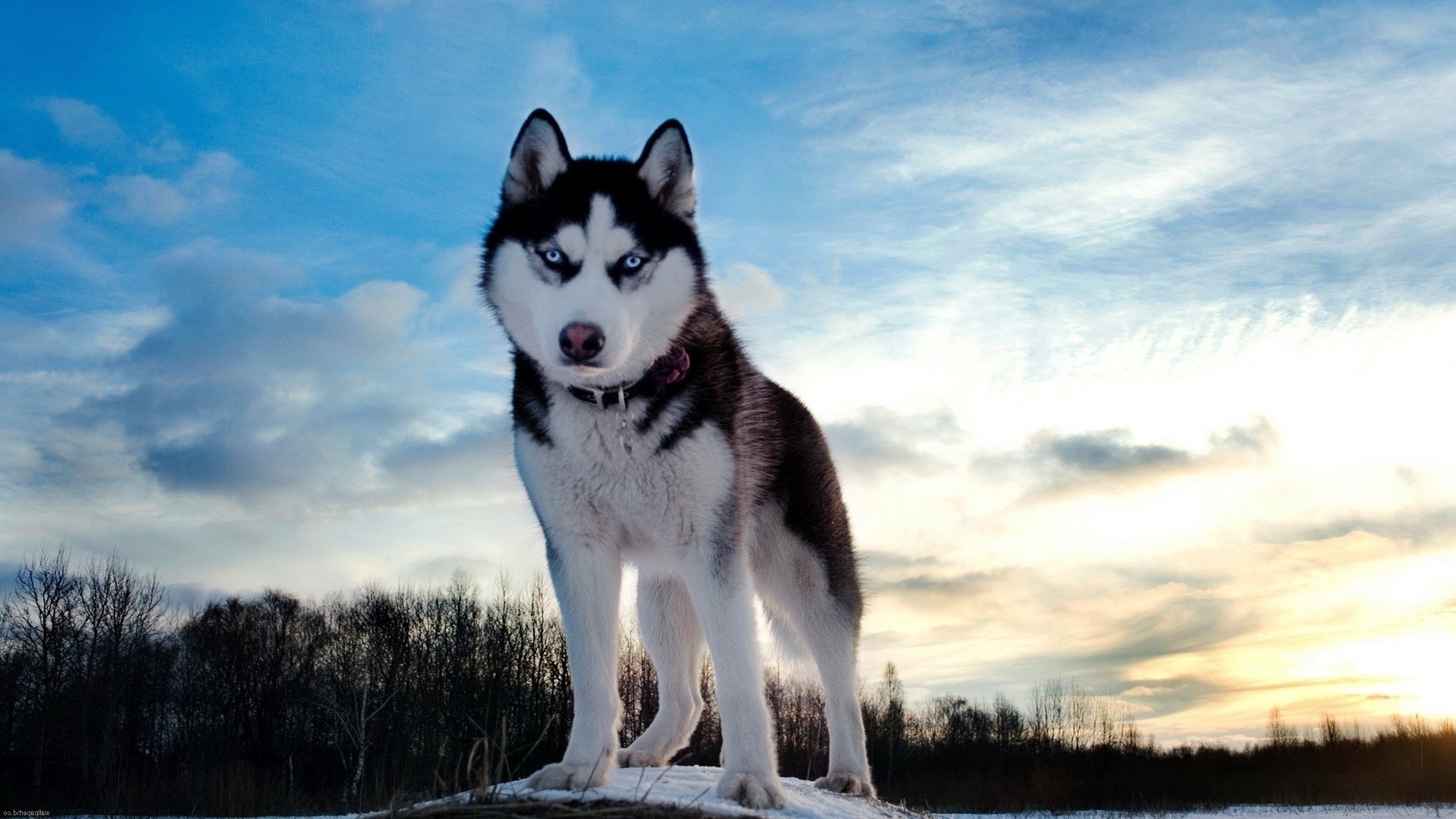}
		\label{fig:husky}
	}
	\caption{Examples for real-world applications of partial label learning. (a) Barcelona's \textcolor[RGB]{41, 50, 137}{Luis Suarez}, \textcolor[RGB]{41, 50, 137}{Lionel Messi} and \textcolor[RGB]{41, 50, 137}{Neymar} celebrate with the trophy after winning the UEFA Champions League Final. (b) A dog image with three candidate labels \{\textcolor[RGB]{41, 50, 137}{Husky, Malamute, Samoyed}\} , which are annotated by the web users, and \textcolor[RGB]{41, 50, 137}{Husky} is the ground-truth label.}
	\label{fig:intro}
\end{figure}

Partial label learning aims to train a multi-class classifier with partially labeled training examples, then utilize the classifier to automatically predict the ground-truth label for an unknown sample. The key challenge of PLL is that the learning algorithms cannot get access to the ground-truth label of training examples \cite{cour2011learning,wang2018towards,zhang2016partial}.
To mitigate this trouble, the most widely used strategy is disambiguation, i.e., identifying the ground-truth label from the candidate label set.

Most existing methods perform disambiguation for candidate labels based on prior knowledge.
Among them, some methods utilize the $\ell_1$ regularization term to capture the noisy labels based on the sparsity assumption\cite{lyu2020hera};
some methods assume that candidate labels are always instance-dependent (feature-dependent)\cite{xu2021instance};
some other works recover the ground-truth label matrix by maintaining the local consistency in the label space \cite{xu2019partial}.
The manifold consistency encourages the similar instances to have the same labels.\cite{xu2019partial,zhang2015solving}.
There exist some other methods exploiting the outputs of model for recovering the true label information.
These methods usually use the model outputs as a guidance to progressively identify the ground-truth labels. For instance, the method proposed in \cite{lv2020progressive} re-weights the losses by the confidences of each candidate label. Some methods encourage the model to output a shaper confidence distribution so as to easily identify the most probable class label\cite{nguyen2008classification,feng2019sure,yu2016maximum}.

Despite the improvement of performance for PLL that these methods have achieved, there still exist two main challenges for learning a multi-class classifier on PLL datasets. On one hand, existing methods rely on prior knowledge, which is often unavailable in practical applications. For instance, the candidate set may not be sparse but composed of many noisy labels in practice, since the training examples may be seriously corrupted in extreme cases. Besides, in high-dimensional feature space, the smoothing assumption may lose effectiveness since it relies on the Euclidean distance.
In these cases, the performance of the model often noticeably decreases due to the lack of prior information.
On the other hand, the disambiguation methods based on model outputs often suffer from the over-fitting issue of noise labels hidden in the candidate set, particularly when there are a large number of noisy labels. The phenomenon often occurs at the early stage, since the model fails to obtain a desirable disambiguation performance at this time due to its insufficient training.

In this paper, we propose a more reliable and effective approach to perform disambiguation for candidate labels of training examples without any additional auxiliary assumptions.
Specifically, we try to recover the true label information by designing a meta-objective on a very tiny validation set (a mini-batch), which can be collected in many practical tasks with small labeling cost. On one hand, we train a classification model by minimizing a confidence-weighted cross entropy loss..
On the other hand, the confidence of each candidate label is adaptively estimated in a meta fashion based on its meta-loss on a tiny validation set.
Theoretically, we show that the model learned by the proposed method is always better than directly learning from partial-labeled data, and we further provide the estimation error bounds for the proposed method. Comprehensive experimental results validate that demonstrate that the proposed method shows superiority to the comparing methods on multiple synthetic and real-world datasets.

The rest of this article is organized as follows. Firstly, we overview the related work in the following section. Then the proposed method will be introduced in Section \ref{sec:approach}. Section \ref{sec:theoretical} demonstrates theoretical properties of the proposed approach. Section \ref{sec:experiment} presents the experimental results, followed by the conclusion in Section \ref{sec:conclusion}.

\section{Related Work}
\label{sec:related work}
Partial label learning aims to train a classifier model merely using ambiguously labeling training data, thereby reducing labeling cost in practice.
The major difficulty of PLL is that the learning algorithms cannot directly get access to the ground-truth label of training examples since it is hidden in the candidate set \cite{cour2011learning,wang2018towards,zhang2016partial}.
To mitigate this issue, a large number of methods have been developed to handle partial labels by adopting the disambiguation strategy, which aims to identify the true label in the candidate set.
Existing PLL methods can be roughly divided into three groups: average-based methods, identification-based methods, and disambiguation-free methods.

Averaging-based methods commonly treat all the candidate labels equally and distinguish the ground-truth label by averaging their modeling outputs.
Cour \textit{et al}.\cite{cour2011learning} proposed to decompose the partial label learning task into a series of binary classification and
then employ SVM to solve them, which is a convex optimization approach named CLPL.
H{\"u}llermeier and Beringer \cite{hullermeier2006learning} proposed a k-nearest neighbors method named PL$k$NN, which distinguishes the ground-truth label by voting among the candidate labels of each neighboring sample.
Zhang and Yu \cite{zhang2015solving} proposed an instance-based method called IPAL, which averages the information of the candidate label from the neighboring instances to predict an unknown example.
Tang and Zhang \cite{tang2017confidence} proposed to enhance the disambiguation ability by simultaneously utilizing the confidences of candidate label and the weights of training samples.
Overall, despite this kind of method can be easily implemented, it has the common drawback that the outputs of false positive labels may mislead the model during the iterative training process and decrease the robustness of learning models.

Identification-based methods progressively estimate the confidence of each candidate label, and then try to recover the ground-truth label from the candidate label set.
Specifically, most existing PLL methods typically iteratively refine the confidence of the ground-truth by regarding it as a latent variable.
These methods update model parameters by using the iteratively refined expectation maximization techniques, such as maximum likelihood \cite{grandvalet2004learning,jin2002learning,liu2012conditional}, and maximum margin\cite{nguyen2008classification,yu2016maximum,chai2019large}.
For instance, Nguyen and Caruana \cite{nguyen2008classification} try to maximize the margin between the output of non-candidate labels and the maximal output of candidate label.
Feng and An \cite{feng2019sure} proposed SURE, which imposes the maximum infinity norm regularization on the modeling outputs to conduct pseudo-labeling and train models simultaneously, then automatically identify the ground-truth label with high confidence.
The above methods are usually restricted to linear models, which suffers from the lack of strong learning ability.
To exploit the powerful fitting ability of the deep neural network, Yao \textit{et al}. \cite{yao2020deep} conducted a preliminary study on deep PLL. They proposed a PLL method named $\text{D}^2$CNN \cite{yao2020deep}, they utilize the entropy regularizer to make the prediction more discriminative and employ convolution neural networks (DCNNs) to improve the ability of feature representation.
In Yao \textit{et al}. \cite{yao2020network}, they proposed to train two networks collaboratively by utilizing a network-cooperation mechanism.
Nevertheless, these methods mainly study the discrepancy between the false positive candidate labels and the ground-truth label but neglect the possibility of each candidate label can be the ground-truth label. To mitigate this issue, some methods try to estimate the confidence of each candidate label rather than directly identifying the ground-truth. For example, Lv \textit{et al}. \cite{lv2020progressive} proposed a progressive identification method named PRODEN, which identifies the true labels according to the output of classifier itself.

Different from aforementioned two disambiguation-based methods, disambiguation-free based methods directly learn from partially labeled data by making modifications to existing techniques. For instance, Zhang \textit{et al}. \cite{zhang2017disambiguation} proposed a disambiguation-free PLL method called ECOC, which decomposes the PLL problem into many binary classification problems and uses Error-Correcting Output Codes (ECOC) coding matrix \cite{dietterich1994solving} to fit partially labeled data.

The proposed method performs disambiguation for candidate labels by designing a meta objective to guide the model training. The idea of using a meta objective to re-weight training examples has been widely applied to train a robust model in many real-world scenarios \cite{ren2018meta,ren2018learning,xie2021partial}.
Similar to MAML\cite{finn2017model}, our method takes one gradient descent step on the meta-objective at each iteration. Besides, unlike these meta learning methods, our meta objective guided disambiguation method does not need any extra hyperparameters.

\section{The Proposed Approach}
\label{sec:approach}
In this paper, we consider the problem of ordinary $c$-class classification.
Let $\mathcal{X}=\mathbb{R}^{d}$ be the feature space with $d$-dimensional, $\mathcal{Y}=\{y_1, y_2, \dots, y_c\}$ be the target space with $c$ class labels. Besides, let $\mathcal{D} = \{(\boldsymbol{x}_i, S_i) \mid 1 \leq i \leq N\}$ be the training set with $N$ partial-labeled instances, where each instance $\boldsymbol{x}_i \in \mathcal{X}$ is represented by a $d$-dimensional feature vector $[\boldsymbol{x}_{i1}, \boldsymbol{x}_{i2}, \dots, \boldsymbol{x}_{id}]^\top$, $S_i \subseteq \mathcal{Y}$ is its corresponding candidate label set. Let $\boldsymbol{\phi}(\boldsymbol{x},\theta)$ be the DNN model, where $\theta$ denotes the parameters.

In the following content, we describe our proposed \textbf{M}eta \textbf{o}bjective \textbf{G}uided \textbf{D}isambiguation (MoGD) framework for solving PLL problems in detail. Firstly, we develop a confidence-weighted objective function to handle partial-labeled examples; then, we propose to adaptively estimate the confidence of each candidate label in a meta-learning fashion. To reduce the computational cost, we alternatively update the confidences of candidate labels and the parameters of the neural network by using an online approximation strategy.

\subsection{Confidence-Weighted Objective Function}

To solve PLL problems, the pioneer work \cite{cour2011learning} derives the partial cross entropy (PCE) loss, which computes the cross entropy (CE) loss for each candidate label. Formally, let $\hat{\y}_i = \phi(\x_i, \theta)$ denote the model output for instance $\x_i$, then we define the PCE loss as:
\begin{equation}
\label{eq:pce}
\mathcal{L}(\mathcal{D}, \theta) = \frac{1}{N} \sum_{i=1}^{N} \ell_{\text{CE}}(\hat{\y}_i, \boldsymbol{y}_i) = - \frac{1}{N} \sum_{i=1}^{N} \sum_{j=1}^{c} y_{ij} \log\hat{y}_{ij},
\end{equation}
where $\y_i=[y_{i1},...,y_{ic}]^\top$ denotes the labels vector, $y_{ij}=1$ if $j\in S_i$; $y_{ij}=0$, otherwise. For notational simplicity, we denote $\boldsymbol{h}_i(\theta)=[h_{i1}(\theta),...,h_{ic}(\theta)]$ as the loss vector computed on $\boldsymbol{x}_i$. Accordingly, the PCE loss can be rewritten by $\mathcal{L}(\mathcal{D}, \theta) = \frac{1}{N} \sum_{i=1}^{N}\sum_{j=1}^c h_{ij}(\theta)$.

Unfortunately, PCE loss regards all candidate labels as ground-truth labels and thus introduces noisy labels into model training. These noisy labels often mislead the model, which leads to a noticeable decrease in generalization performance. To mitigate the negative influence of noise labels, a label confidence vector ${\boldsymbol{w}_i} = [w_{i1}, w_{i2}, \dots, w_{ic}]$ is assigned to every example $\boldsymbol{x}_i$, where the confidence $w_{ij}\in[0,1]$ is use to estimate the probability that the
example $\boldsymbol{x}_i$ belongs to $j$-th class.
Obviously, if $y_{ij}$ is not a candidate label, then we keep its confidence $w_{ij}$ being 0 due to the fact that it cannot be a ground-truth label for example $\boldsymbol{x}_i$.

The confidence-weighted cross entropy loss can be formulated as follows:
\begin{equation}\label{eq:train}
	\mathcal{L}^{tr}(\mathcal{D}, \theta) = \frac{1}{N} \sum_{i=1}^{N} \boldsymbol{w}_{i} \boldsymbol{h}_i(\theta) = \frac{1}{N}\sum_{i=1}^{N}\sum_{j=1}^{c}w_{ij}h_{ij}(\theta),
\end{equation}
In the ideal case, if we can accurately recover the confidences for ground-truth labels, then $\mathcal{L}^{tr}$ would be degenerated into the standard cross entropy, which yields that the disambiguation is achieved for candidate labels.

Nevertheless, it is hard to estimate the confidence precisely without any prior knowledge in many real-world scenarios. To mitigate this issue, we follow a meta-learning paradigm to design a disambiguation strategy based on a meta objective.

\subsection{Meta Objective Guided Disambiguation}

In this subsection, we aim to estimate a optimal confidence distribution for every example in an meta-learning fashion. Towards this goal, we firstly measure the performance of confidence estimation based on a validation set, and then for each candidate label, we use the measurement as a guidance to adaptively estimate its confidence. The only need is a tiny clean validation set (e.g., a mini-batch of examples), which can be collected in many realistic tasks with small labeling cost.

Specifically, we denote $\mathcal{D}_{val} = \{(\boldsymbol{x}_i^v,\y_i^v) \mid 1 \le i \le M \}$ as a validation set with $M\ (M \ll N)$ examples available during the training phase. Here, for every validation example $(\x^v_i,\y^v_i)$, we use superscript $v$ to distinguish it from training examples. It is noteworthy that every validation example $\x^v_i$ has been annotated with its ground-truth label vector $\y^v_i$. Intuitively, \textit{the optimal confidences often leads to the best model such that minimizes the validation loss}. Based on this intuition, we can treat the objective function on the validation set as guidance to optimize the confidences. The meta objective function can be formulated as follows:
\begin{equation}
	\label{eq:val_loss}
	\mathcal{L}^{meta}(\mathcal{D}_{val}, \theta^*(\boldsymbol{w})) = \frac{1}{M} \sum_{i=1}^{M}\ell_\text{CE}(\hat{\y}_i^v, \y^v_i),
\end{equation}
where $\mathcal{L}^{meta}$ represents the loss of trained classifier $\theta^*(\w)$ on the validation set $\mathcal{D}_{val}$.
In Section \ref{sec:experiment}, we discuss that the proposed method is influenced by the size of validation set, and the results show that the performance of MoGD is insensitive to the size of validation set.

\begin{algorithm}[!t]
	\caption{Meta Objective Guided Disambiguation for Partial Label Learning.}
	\label{alg:pllmd}
	\begin{algorithmic}
		\STATE
		\STATE {\textsc{INPUT:}}
		\STATE \hspace{0.5cm} The training set $\mathcal{D} = \{(\boldsymbol{x}_i, S_i )\mid 1 \le i \le N\}$;
		\STATE \hspace{0.5cm} The validation set $\mathcal{D}_{val} = \{(\boldsymbol{x}_i^v, S_i^v) \mid 1 \le i \le M\}$;
		\STATE \hspace{0.5cm} The max iteration $T$;
		\STATE {\textsc{PROCEDURE:}}
		\STATE \hspace{0.5cm} Initialize the net parameters $\theta$ and confidences $\boldsymbol{w}$;
		\STATE \hspace{0.5cm} \textbf{for} $t=1$ \textbf{to} $T$ \textbf{do}:
		\STATE \hspace{1cm} Sample a mini-batch examples $\mathcal{B} = \{ (\x_i, S_i)\}_{i=1}^b$
		\STATE \hspace{1cm} from training set $\mathcal{D}$;
		\STATE \hspace{1.cm} Update $\hat{\theta}^{(t)}$ according to Eq.\eqref{eq:theta_hat};
		\STATE \hspace{1.cm} Compute $\mathcal{L}^{meta}$ according to Eq.\eqref{eq:val_loss};
		\STATE \hspace{1.cm} Update $\w^{(t)}$ according to Eq.\eqref{eq:w_hat}, Eq.\eqref{eq:w_clip}, Eq.\eqref{eq:w_normalize};
		\STATE \hspace{1.cm} Update $\theta^{(t+1)}$ according to Eq.\eqref{eq:theta_t1};
		\STATE \hspace{0.5cm} \textbf{end for};
		\STATE {\textsc{OUTPUT:}}
		\STATE \hspace{0.5cm} The classifier model $\boldsymbol{\phi}$;
	\end{algorithmic}
\end{algorithm}

We can adopt the alternative strategy to optimize variables $\theta$ and $\w$ by updating one of the variables while fixing the other. The process continues until both of these two variables converge or exceed the maximum iteration. Unfortunately, it requires two nested loops to obtain optimal parameters, which is unbearable in many realistic tasks, especially when the deep model is used. To improve the training efficiency, inspired by the previous works \cite{ren2018learning}, the online approximation strategy is used to obtain the optimal variables $\theta^*$ and $\w^*$.  Specifically, to obtain the $\w^*$, motivated by the EM algorithm, we can initialize the confidences uniformly as follows:
\begin{equation}
	w_{ij} = \begin{cases}
		\frac{1}{|S_i|}& j \in S_i\\
		0 ,& \text{otherwise}
	\end{cases}
\end{equation}
Then, at every training iteration $t$, given a mini-batch training samples $\mathcal{B} = \{ (\boldsymbol{x}_i, S_i)\}_{i=1}^b$, where $b$ denotes the mini-batch size, that satisfies $b\ll N$. Then we can update the parameters $\theta$ based on the descent direction of the objective function on this mini-batch:
\begin{equation}
\label{eq:theta_hat}
\hat{\theta}^{(t)}(\boldsymbol{w}) = \theta^{(t)} - \left.\alpha \nabla_\theta \mathcal{L}_{tr}(\mathcal{D}_b, \theta)\right|_{\theta = \theta^{(t)}},
\end{equation}
where $\alpha$ is the step size with respect to $\theta$.

Next we can seek the optimal confidence $\boldsymbol{w}^*$ by minimizing the validation loss $\mathcal{L}^{meta}$.

Based on the $\hat{\theta}^{(t)}$, we can take a single gradient descent step according to the direction of the meta-objective on the mini-batch validation samples to get a cheap estimate of $\boldsymbol{\hat{w}}^{(t)}$, which can be formally written as:
\begin{equation}
\label{eq:w_hat}
\boldsymbol{\hat{w}}^{(t)} =  -\left.\beta \nabla\boldsymbol{_w} \mathcal{L}^{meta}(\mathcal{D}_{val}, \hat{\theta}^{(t)}(\boldsymbol{w})) \right|_{\boldsymbol{w}=0},
\end{equation}
where $\beta$ is the step size with respect to $\boldsymbol{w}$. To make sure the the expect loss non-negative, it requires us to guarantee each $w_{ij} \in \boldsymbol{w}$ is non-negative:
\begin{equation}
\label{eq:w_clip}
\hat{w}_{ij}^{(t)} = \max(\hat{w}_{ij}^{(t)}, 0), \forall\ i \in [1, m] \text{ and } j \in [1, c].
\end{equation}

In practice, for each training example, we normalize the confidences with respect to every class so that they sum to one:
\begin{equation}
\label{eq:w_normalize}
w_{ij}^{(t)} = \begin{cases}
\exp(\hat{w}_{ij}^{(t)}) / \sum_{j\in S_i} \exp(\hat{w}_{ij}^{(t)}) ,& j \in S_i\\
0 ,& \text{otherwise}
\end{cases}
\end{equation}
Finally, we can optimize the model parameter $\theta^{(t)}$ by employing the gradient descent based on the latest updated confidences $\boldsymbol{w}^{(t)}$, which is described below:
\begin{equation}
\label{eq:theta_t1}
\theta^{(t+1)} = \theta^{(t)} - \left.\alpha \nabla_\theta\mathcal{L}^{tr}(\mathcal{D}_b, \theta(\boldsymbol{w}^{(t)})) \right|_{\theta = \theta^{(t)}}.
\end{equation}
We repeat these procedures until both of them converge or exceed the maximum epoch.
The step-by-step pseudo-code of MoGD is shown in Algorithm \ref{alg:pllmd}.

\section{Theoretical Analysis}
\label{sec:theoretical}
\newenvironment{proof}{{\par\it Proof:}\quad}{\hfill $\square$\par}
\newtheorem{Definition}{Definition}
\newtheorem{Theorem}{Theorem}
\newtheorem{Lemma}{Lemma}
In this section, we first analyze the convergence of MoGD, then establish the estimation error bounds.
\subsection{Convergence}
Different from optimization on a single-level problem, the proposed methods is a optimization problem with two-level objectives including Eq.\eqref{eq:w_hat} and Eq.\eqref{eq:theta_t1}. Following the analysis procedure in \cite{ren2018learning}, we obtain the convergence results that both the training and meta objective function converges to the critical points under some mild conditions.

\begin{Definition}
	A function $f(x): \mathbb{R}^{d} \rightarrow \mathbb{R}$ is said to be Lipschitz-smooth with constant $L$ if \[\|\nabla f(x)-\nabla f(y)\| \leq L\|x-y\|, \forall x, y \in \mathbb{R}^{d}\]
\end{Definition}
\begin{Definition}
	$f(x)$ has $\sigma$-bounded gradients if $\Vert \nabla f(x)\Vert \le \sigma $ for all $x\in \mathbb{R}^d$.
\end{Definition}
\begin{Theorem}[convergence]
	Suppose that the meta loss function is Lipschitz-smooth with constant $L_\ell$, and for training data $\boldsymbol{x}_i$, the training loss function $\ell_i$ have $\sigma$-bounded gradients. Then the validation loss always monotonically decreases with the iteration $t$ by employing our optimization method, i.e.,
	\begin{equation*}
		\mathcal{L}^{meta}(\theta^{(t+1)}) \le \mathcal{L}^{meta}(\theta^{(t)}).
	\end{equation*}
	Note that, the step size $\alpha^{(t)}$ satisfies $\alpha^{(t)} \le \frac{2nc}{L_\ell\sigma^2}$, where $n$ denotes the batch size of training, and $c$ denotes the number of classes.

	Furthermore, if the gradient of validation loss becomes $0$ at iteration $t$, the equality will hold, i.e.,
	\[\mathcal{L}^{meta}(\theta^{(t+1)}) = \mathcal{L}^{meta}(\theta^{(t)})\]
	if and only if
	\[\nabla\mathcal{L}^{meta}(\theta^{(t)}) = 0.\]
\end{Theorem}


\subsection{Comparisons with Ordinary Methods}
In the following content, we discuss the effectiveness of the proposed approach compared with ordinary PLL methods.

First, define the empirical risk as:
\begin{equation*}
	\hat{R}(\theta) = \frac{1}{N} \sum_{i=1}^{N} \ell (\phi(\boldsymbol{x}_i, \theta), \boldsymbol{y}_i),
\end{equation*}
then we have the theorem below.
\begin{Theorem}[effectiveness]
	Define $\theta'$ is the model trained without using validation set, i.e.,
	\begin{equation*}
		\theta' = \arg \min_{\theta \in \Theta} \sum_{i=1}^N \ell(\boldsymbol{\phi}(\boldsymbol{x}_i, \theta), \boldsymbol{y}_i).
	\end{equation*}
 	Then the empirical risk of $\hat{\theta}$ that produced by MoGD is never worse than $\theta'$ which trained merely utilizes original partial labeled examples, i.e., $\hat{R}(\hat{\theta}) \le \hat{R}(\theta')$.
\end{Theorem}
\begin{proof}
	First, we have
	\[\hat{\theta} = \arg \min_{\theta \in \Theta} \sum_{i=1}^{N} \sum_{j=1}^{c}w_{ij} h_{ij}(\theta))\]
	Suppose $\hat{R}(\hat{\theta}) > \hat{R}(\theta')$, obviously we can always set the weight of each label to its initial value, e.g., non-candidate label to zero and candidate label to one. Then we can obtain the $\hat{R}(\hat{\theta}) = \hat{R}(\theta')$.
	Therefore, $\hat{R}(\hat{\theta})$ is never worse than $\hat{R}(\theta')$.
\end{proof}
\begin{table*}[htbp]
	\caption{Characteristics of Three Image Datasets and Models.}
	\label{tab:char cifar}
	\centering
	\begin{tabular}{c|c c c c|c}
		\toprule
		Dataset & \# Train & \# Test & \# Feature& \# Class & Model $\phi(x,\theta)$\\
		\midrule
		Fashion-MNIST & 60000 & 10000 & 784 &10 & Linear Model\\
		CIFAR-10 & 50000 & 10000 & 3072 & 10 &ResNet, ConvNet \\
		CIFAR-100 & 50000 & 10000 & 3072 & 100 &  ResNet, ConvNet\\
		\bottomrule
	\end{tabular}
\end{table*}
\begin{table*}[!tb]
	\caption{Characteristics of Five Real-World Datasets and Models.}
	\label{tab:char real world}
	\centering
	\begin{tabular}{c|c c c c c|c}
		\toprule
		Dataset & \# Instance & \# Feature & \# Avg.Label & \# Class & Domain & Model $\phi(x,\theta)$\\
		\midrule
		Lost & 1122 &108 &2.23 &16 & automatic face naming & Linear Model\\
		BirdSong & 4998 & 38 & 2.18 & 13 & bird song classification & Linear Model\\
		MSRCv2 & 1758 & 48 & 3.16 & 23 & object classification & Linear Model\\
		Soccer Player& 17472& 279& 2.09& 171& automatic face naming & Linear Model\\
		Yahoo! News & 22991 & 163 & 1.91 & 219 & automatic face naming & Linear Model\\
		\bottomrule
	\end{tabular}
\end{table*}
\subsection{Estimation Error Bounds}
In this subsection, we establish the estimation error bounds for the proposed method.

We begin with several useful definitions and lemmas as follows.
\begin{Definition}[$\epsilon$-cover]
	For a set $\mathcal{A}$, if for $\forall \alpha \in \mathcal{B}, \exists \alpha^{\prime} \in \mathcal{A}$ satisfies $\left\|\alpha-\alpha^{\prime}\right\| \leq \epsilon$, then $\mathcal{A}$ is an $\epsilon$-cover of $\mathcal{B}$.
\end{Definition}
\begin{Definition}[Rademacher Complexity]
	Define $p(x,y)$ is the underlying joint density of random variables $(\x,\y) \in \mathcal{X} \times \mathcal{Y}$.
	Let $\mathcal{G} = \{ g(\x)\}$ be a family of functions for empirical risk minimization, $\sigma_1, \dots, \sigma_n$ be $n$ Rademacher variables, then we can define the Rademacher complexity of $\mathcal{G}$ over $p(x, y)$ with example size $n$ as follows\cite{mohri2018foundations}:
	\begin{equation*}
		\Re_n(\mathcal{G})=\mathbb{E}_{\sigma_1, \dots, \sigma_n} \left[ \sup_{g\in \mathcal{G}} \frac{1}{n} \sum_{\x_i\in \mathcal{X}} \sigma_i g(\x_i)\right];
	\end{equation*}
\end{Definition}

Assume the meta-objective function $\ell$ is upper-bounded by $K$ and Lipschitz-smooth with a constant $L_\ell$, i.e.,
\[K = \sup_{\x \in \mathcal{X}, \y \in \mathcal{Y}, g\in \mathcal{G}}\ell(g(\x), \y).\]

For notational simplicity, we use $\hat{R}$ and $R$ to denote empirical risk and generalization risk, respectively.
Following prior works\cite{lv2020progressive,ren2018learning,guo2020safe,ishida2017learning}, we summarized two lemmas as follows.

\begin{Lemma}(\cite{lv2020progressive,ishida2017learning})
	\label{lemma:1}
	For any $\delta > 0$, with a probability at least $1-\delta$, we have
	\begin{equation*}
		\sup_{g \in \mathcal{G}} \left| R(g) - \hat{R}(g) \right|  \le 2\Re_n(\ell \circ \mathcal{G}) + K \sqrt{\frac{\log(2/\delta)}{2n}}.
	\end{equation*}
\end{Lemma}

\begin{Lemma}(\cite{lv2020progressive,ishida2017learning})
	\label{lemma:2}
	Suppose the loss function $\ell$ is defined as above section, we have
	\begin{equation*}
		\Re_n(\ell \circ \mathcal{G}) \le \sqrt{2}cL_{\ell}\Re_n(\mathcal{G}),
	\end{equation*}
	where $L_\ell$ is the Lipschitz constant of the loss function $\ell$.
\end{Lemma}

Based on Lemma \ref{lemma:1} and \ref{lemma:2}, we can establish the estimation error bounds as follows.

\begin{Theorem}[estimation error bounds]
	\label{theorem:eb}
	Let $\boldsymbol{w} \in \mathbb{B}^c$ be the parameter of label confidence of the examples, which is bounded in a $c$-dimensional unit ball. Let
	Then, the generalization risk can be defined as:
	\[R(\theta)=\mathbb{E}_{p(\x, \y)}[\ell(\phi(\x;\theta), \y)].\]
	Let $\hat{\w} =\arg \min _{\boldsymbol{w} \in \mathcal{A}} \hat{R}(\hat{\theta}(\boldsymbol{w}))$ be the empirically optimal parameter in a candidate set $\mathcal{A}$, and let $\w^*=\arg \min _{\boldsymbol{w} \in \mathbb{B}^{c}} R(\hat{\theta}(\boldsymbol{w}))$ be the optimal parameter in the unit ball.
	Then, with a probability at least $1-\delta$, we have,
	\begin{equation*}
		R(\hat{\theta}(\hat{\w})) - R(\hat{\theta}(\w^*)) \le 4\sqrt{2}cL_\ell\Re(\mathcal{G}) + K\sqrt{\frac{2\log(2/\delta)}{n}},
	\end{equation*}
	where $n$ denotes the number of examples.

\end{Theorem}
\begin{proof}
	\begin{align*}
		R(\hat{\theta}(\hat{\w})) - R(\hat{\theta}(\w^*)) & = \left(\hat{R}(\hat{\theta}(\hat{\w})) - \hat{R}(\hat{\theta}(\w^*))\right) \\
														  & + \left(R(\hat{\theta}(\hat{\w})) - \hat{R}(\hat{\theta}(\hat{\w})) \right) \\
														  &	+ \left(\hat{R}(\hat{\theta}(\w^*)) - R(\hat{\theta}(\w^*)) \right) \\
														  & \le 0 + 2 \max|\hat{R}(\hat{\theta}(\alpha)) - R(\hat{\theta}(\alpha))| \\
														  & \le 4\sqrt{2}cL_\ell\Re(\mathcal{G}) + K\sqrt{\frac{2\log(2/\delta)}{n}},
	\end{align*}
	where we used that $\hat{R}(\hat{\theta}(\w^*)) > \hat{R}(\hat{\theta}(\hat{\w}))$ by the definition of $\w^*$ and $\hat{\w}$.
\end{proof}

Theorem \ref{theorem:eb} guarantees that: $n\rightarrow\infty$, $R(\hat{\theta}(\hat{\w})) \rightarrow R(\hat{\theta}(\w^*))$.

\begin{table*}[!tb]
	\caption{Accuracy (Mean$\pm$Std) Comparison Results on Fashion-MNIST with \textbf{Uniform} Partial Labels on Different Ambiguity Levels, the Backbone is \textbf{Linear} Model. The Best Results are Shown in Bold.}
	\label{tab:result fashion}
	\centering
	\begin{tabular}{c|c| c c c c}
		\toprule
		Dataset & Method & $q=0.1$ & $q=0.3$ & $q=0.5$ & $q=0.7$\\
		\midrule
		\multirow{11}{*}{Fashion-MNIST}
		& MoGD & $\mathbf{85.74 \pm 0.09\%}$ & $\mathbf{85.62 \pm 0.15\%}$ &  $\mathbf{85.13 \pm 0.18\%}$ & $\mathbf{83.87 \pm 0.20\%}$  \\
		& VALEN & $82.33 \pm 0.24 \%$ & $81.72 \pm 0.13 \%$ & $80.93 \pm 0.36 \%$ & $78.46\pm 0.32\%$ \\
		& LWS & $84.62 \pm 0.14 \%$ & $84.26 \pm 0.18 \%$ & $83.98 \pm 0.08 \%$ & $82.47\pm 0.21\%$ \\
		& PRODEN & $84.16 \pm 0.08 \%$ & $83.94 \pm 0.12 \%$ & $83.87 \pm 0.06 \%$ & $82.15\pm 0.26\%$  \\
		& RC & $84.53 \pm 0.11 \%$ & $84.16 \pm 0.15 \%$ & $83.95 \pm 0.17 \%$ & $82.64\pm 0.33\%$ \\
		& CC & $84.42 \pm 0.18 \%$ & $84.22 \pm 0.06 \%$ & $83.84 \pm 0.10 \%$ & $82.88\pm 0.30\%$  \\
		& MSE & $82.21 \pm 0.14 \%$ & $81.02 \pm 0.18 \%$ & $77.25 \pm 0.46 \%$ & $70.47\pm 0.51\%$  \\
		& EXP & $83.01 \pm 0.15 \%$ & $80.54 \pm 0.23 \%$ & $79.09 \pm 0.37 \%$ & $71.12\pm 0.65\%$  \\
		& \cellcolor[HTML]{D9D9D9}Fully Supervised&  \multicolumn{4}{c}{\cellcolor[HTML]{D9D9D9} $86.21 \pm 0.13 \%$}  \\
		\bottomrule
	\end{tabular}
\end{table*}
\begin{table*}[!tb]
	\caption{Accuracy (Mean$\pm$Std) Comparison Results on CIFAR-10 and CIFAR-100 with \textbf{Uniform} Partial Labels on Different Ambiguity Levels, the Backbone is \textbf{ConvNet} Model. The Best Results are Shown in \textbf{Bold}.}
	\label{tab:result cifar cnn}
	\centering
	\begin{tabular}{c|c| c c c c}
		\toprule
		Dataset & Method & $q=0.1$ & $q=0.3$ & $q=0.5$ & $q=0.7$\\
		\midrule
		\multirow{11}{*}{CIFAR-10}
		&MoGD & $\mathbf{92.67 \pm 0.08\%}$ & $\mathbf{91.64 \pm 0.09\%}$ &  $\mathbf{88.92 \pm 0.07\%}$ & $\mathbf{85.11 \pm 0.14\%}$  \\
		& VALEN & $75.81 \pm 0.15 \%$ & $71.68 \pm 0.23 \%$ & $63.85 \pm 0.36 \%$ & $54.62\pm 0.86\%$ \\
		& LWS & $87.63 \pm 0.17 \%$ & $84.65 \pm 0.28\%$ & $68.29 \pm 0.31 \%$ & $55.57 \pm 0.65\%$ \\

		& PRODEN & $90.22 \pm 0.04 \%$ & $88.56 \pm 0.08\%$ & $84.48 \pm 0.03 \%$ & $81.58\pm 0.09\%$  \\
		& RC & $88.08 \pm 0.05 \%$ & $86.91 \pm 0.10 \%$ & $79.54 \pm 0.12 \%$ & $70.08 \pm 0.17\%$ \\
		& CC & $88.24 \pm 0.09 \%$ & $85.72 \pm 0.08 \%$ & $74.14 \pm 0.18 \%$ & $67.26 \pm 0.21\%$  \\
		& MSE & $78.32 \pm 0.13 \%$ & $71.47 \pm 0.30 \%$ & $65.18 \pm 0.16 \%$ & $60.43\pm 0.24\%$  \\
		& EXP & $77.64 \pm 0.15 \%$ & $72.38 \pm 0.26 \%$ & $68.94 \pm 0.14 \%$ & $62.19\pm 0.26\%$  \\
		& \cellcolor[HTML]{D9D9D9}Fully Supervised&  \multicolumn{4}{c}{\cellcolor[HTML]{D9D9D9} $93.74 \pm 0.12 \%$}  \\
		\midrule
		Dataset & Method & $q=0.01$ & $q=0.05$ & $q=0.1$ & $q=0.2$\\
		\midrule
		\multirow{11}{*}{CIFAR-100}
		&MoGD & $\mathbf{72.14 \pm 0.23\%}$ & $\mathbf{70.26 \pm 0.34\%}$ &  $\mathbf{65.23 \pm 0.41\%}$ & $\mathbf{56.84 \pm 1.26\%}$  \\

		& VALEN & $52.64 \pm 0.20 \%$ & $51.32 \pm 0.34 \%$ & $44.38 \pm 0.23 \%$ & $12.68\pm 0.72\%$ \\
		& LWS & $59.74 \pm 0.31 \%$ & $57.21 \pm 0.28\%$ & $46.52 \pm 0.51 \%$ & $30.26 \pm 0.63\%$ \\

		& PRODEN & $67.51 \pm 0.05 \%$ & $63.69 \pm 0.14 \%$ & $54.79 \pm 0.12 \%$ & $44.36\pm 0.58\%$  \\
		& RC & $61.59 \pm 0.15 \%$ & $53.86 \pm 0.23 \%$ & $44.60 \pm 0.38 \%$ & $33.62\pm 0.84\%$ \\
		& CC & $60.72 \pm 0.21 \%$ & $52.94 \pm 0.18 \%$ & $45.85 \pm 0.61 \%$ & $39.73\pm 0.95\%$  \\
		& MSE & $46.74 \pm 0.15 \%$ & $43.16 \pm 0.21 \%$ & $40.18 \pm 1.24 \%$ & $6.88\pm 0.16\%$  \\
		& EXP & $42.56 \pm 0.18 \%$ & $38.85 \pm 0.33 \%$ & $26.72 \pm 2.47 \%$ & $6.14\pm 0.22\%$  \\
		& \cellcolor[HTML]{D9D9D9}Fully Supervised&  \multicolumn{4}{c}{\cellcolor[HTML]{D9D9D9} $73.56 \pm 0.08 \%$}  \\
		\bottomrule
	\end{tabular}
\end{table*}
\section{Experiments}
\label{sec:experiment}
In this section, we conduct experiments on various datasets, and compare it with some state-of-the-art partial label learning algorithms to validate the effectiveness of MoGD.

\subsection{Experiment Setting}
\subsubsection{Datasets}
We perform experiments on three widely used benchmark image datasets, including Fashion-MNIST\footnote{\url{https://github.com/zalandoresearch/fashion-mnist}}\cite{xiao2017fashion},
CIFAR-10 and CIFAR-100 \footnote{\url{https://www.cs.toronto.edu/~kriz/cifar.html}} \cite{krizhevsky2009learning} as well as five real-world PLL datasets, including Lost \cite{cour2011learning}, BirdSong \cite{briggs2012rank}, MSRCv2 \cite{liu2012conditional}, Soccer Player \cite{zeng2013learning}, and Yahoo! News \cite{guillaumin2010multiple}.
Table \ref{tab:char cifar} and Table \ref{tab:char real world} summarize the detailed characteristics of these datasets, respectively.

Note that it requires to manually generate the candidate label sets for three image datasets, since they are supposed to be used for multi-class classification. We manually generate candidate labels for training examples from these datasets based on two different generation process, i.e., uniform generation process \cite{lv2020progressive} and instance-dependent generation process \cite{xu2021instance}.

Specifically, we generate candidate labels by flipping a negative label $(y=0)$ into false positive label $\tilde{y}$ with a flipping probability $q$, where $q=P(\tilde{y}=1 \mid y = 0)$. In other words, we employ a binomial flipping strategy: for each training example, we conduct $c-1$ independent Bernoulli experiments, and each experiment determines whether a false positive label is generated with probability $q$. To make sure all training example have two or more candidate labels, for the instance that there is none false positive labels generated, we randomly select a negative label and add it to the candidate label set. Similar to previous works\cite{lv2020progressive,feng2020provably}, we consider $q\in\{0.1,0.3,0.5,0.7\}$ for Fashion-MNIST and CIFAR-10, and $q\in\{0.01,0.05, 0.1, 0.2\}$ for CIFAR-100. Generally, a lager value of $q$ indicates a higher level of ambiguity of training examples.

To construct instance-dependent candidate labels, we follow the generating procedure used in \cite{xu2021instance}. Specifically, for an instance $\x_i$, let $q_j(\x_i)=p(\tilde{\y_j}=1|\y_j=0,\x_i)$ be the flipping probability of $j$-th label, where $\y$ denotes the original label vector.
Then we use the prediction probability of a clean network $g$ trained on the clean labels to set the flipping probability of each negative label with $q_j(\x_i) = \frac{g_j(\boldsymbol{\x_i})}{\max_{k\in \bar{Y}_i}g_k(\boldsymbol{\x_i})}$, where $\bar{Y}_i$ denotes the negative label set of $\x_i$.
In our experiments, we directly adopt the pre-trained model released by \cite{xu2021instance} to generate instance-dependent candidate labels.

\subsubsection{Compared methods}
In order to validate the effectiveness of the proposed method, we first compare MoGD against the following five deep PLL methods on three image datasets:
\begin{itemize}
	\item VALEN\cite{xu2021instance}: An instance-dependent method which uses variational inference technique to iteratively estimate the latent label distribution during the training stage.
	\item LWS\cite{wen2021leveraged}: A discriminative method which imposes a trade-off between losses on non-candidate labels and candidate labels.
	\item PRODEN\cite{lv2020progressive}: A progressive identification method which identifies the true labels according to the output of classifier itself.
	\item RC\cite{feng2020provably}: A risk-consistent method which establishes the true risk estimator by employing the importance re-weighting strategy.
	\item CC\cite{feng2020provably}: A classifier-consistent method which establishes the empirical risk estimator by employing the transition matrix and cross entropy loss.
\end{itemize}
The baseline methods:
\begin{itemize}
	\item EXP\cite{feng2020learning} and MSE\cite{feng2020learning} are two simple baselines which utilize exponential loss and mean square error as the risks, respectively.
	\item Fully Supervised: It trains a multi-class classifier based on ordinary labeled data and can be regarded as an upper bound of performance.
\end{itemize}

In addition, we also compare MoGD with eight classical PLL methods on five real-world datasets, including
GA\cite{ishida2019complementary},
$\text{D}^2$CNN\cite{yao2020deep}, SURE\cite{feng2019sure}, CLPL\cite{cour2011learning}, ECOC\cite{zhang2017disambiguation}, PLSVM\cite{nguyen2008classification}, PL$k$NN\cite{hullermeier2006learning}, IPAL\cite{zhang2015solving}. The details of these methods can be found in the related work section. In our experiments, the hyper-parameters are determined based on the recommended parameter settings in their original literature.
\begin{table*}[htbp]
	\caption{Accuracy (Mean$\pm$Std) Comparison Results on CIFAR-10 and CIFAR-100 with \textbf{Uniform} Partial Labels on Different Ambiguity Levels, the Backbone is \textbf{ResNet} Model. The Best Results are Shown in \textbf{Bold}.}
	\label{tab:result cifar resnet}
	\centering
	\begin{tabular}{c|c| c c c c}
		\toprule
		Dataset & Method & $q=0.1$ & $q=0.3$ & $q=0.5$ & $q=0.7$\\
		\midrule
		\multirow{11}{*}{CIFAR-10}
		& MoGD & $\mathbf{86.47 \pm 0.07\%}$ & $\mathbf{86.31 \pm 0.14\%}$ &  $\mathbf{85.92 \pm 0.13\%}$ & $\mathbf{83.46 \pm 0.09\%}$  \\
		& VALEN & $70.52 \pm 0.13 \%$ & $68.37 \pm 0.26 \%$ & $60.85 \pm 0.48 \%$ & $52.73\pm 1.58\%$ \\
		& LWS & $82.16 \pm 0.12 \%$ & $80.84 \pm 0.16\%$ & $63.38 \pm 0.26 \%$ & $54.86 \pm 0.26\%$ \\
		& PRODEN & $85.39 \pm 0.17 \%$ & $84.74 \pm 0.29\%$ & $83.79 \pm 0.12 \%$ & $79.88\pm 0.11\%$  \\
		& RC & $83.69 \pm 0.06 \%$ & $81.58 \pm 0.42 \%$ & $78.92 \pm 0.41 \%$ & $64.48 \pm 1.24\%$ \\
		& CC & $82.48 \pm 0.08 \%$ & $80.44 \pm 0.39 \%$ & $76.36 \pm 0.47 \%$ & $65.32 \pm 1.61\%$  \\
		& MSE & $61.52 \pm 0.94 \%$ & $58.39 \pm 0.74 \%$ & $53.15 \pm 0.92 \%$ & $43.52\pm 2.81\%$  \\
		& EXP & $63.24 \pm 0.82 \%$ & $59.32 \pm 0.87 \%$ & $55.35 \pm 0.64 \%$ & $44.48\pm 2.13\%$  \\
		& \cellcolor[HTML]{D9D9D9}Fully Supervised&  \multicolumn{4}{c}{\cellcolor[HTML]{D9D9D9} $87.32 \pm 0.08 \%$}  \\
		\midrule
		Dataset & Method & $q=0.01$ & $q=0.05$ & $q=0.1$ & $q=0.2$\\
		\midrule
		\multirow{11}{*}{CIFAR-100}
		& MoGD & $\mathbf{68.77 \pm 0.34\%}$ & $\mathbf{64.85 \pm 0.42\%}$ &  $\mathbf{58.46 \pm 0.27\%}$ & $\mathbf{50.33 \pm 0.56\%}$  \\

		& VALEN & $48.82 \pm 0.39 \%$ & $47.65 \pm 0.54 \%$ & $38.66 \pm 0.63 \%$ & $10.45\pm 0.84\%$ \\
		& LWS & $53.46 \pm 0.27 \%$ & $52.36 \pm 0.23\%$ & $40.67 \pm 0.43 \%$ & $28.74 \pm 0.85\%$ \\

		& PRODEN & $57.71 \pm 0.46 \%$ & $53.16 \pm 0.29 \%$ & $50.62 \pm 0.54 \%$ & $40.78\pm 0.65\%$  \\
		& RC & $54.67 \pm 0.18 \%$ & $51.89 \pm 0.31 \%$ & $39.42 \pm 0.12 \%$ & $30.51\pm 0.18\%$ \\
		& CC & $51.91 \pm 0.34 \%$ & $51.83 \pm 0.56 \%$ & $41.57 \pm 0.34 \%$ & $35.13\pm 0.23\%$  \\
		& MSE & $38.56 \pm 0.26 \%$ & $36.61 \pm 0.38 \%$ & $31.27 \pm 0.13 \%$ & $4.82\pm 0.19\%$  \\
		& EXP & $36.37 \pm 0.16 \%$ & $34.73 \pm 0.14 \%$ & $24.45 \pm 0.28 \%$ & $4.25\pm 0.22\%$  \\
		& \cellcolor[HTML]{D9D9D9}Fully Supervised&  \multicolumn{4}{c}{\cellcolor[HTML]{D9D9D9} $70.08 \pm 0.11 \%$}  \\
		\bottomrule
	\end{tabular}
\end{table*}

\subsubsection{Implementation Details}
We employ multiple basic models, including linear model, 12-layer ConvNet\cite{laine2016temporal} and Wide-ResNet-28\cite{zagoruyko2016wide} to show that the proposed approach is compatible with different learning models.
Specifically, for Fashion-MNIST and five real-world datasets, we employ Linear Net as the backbone neural network while using ConvNet and ResNet for CIFAR-10 and CIFAR-100.

To make a fair comparison, we employ the same network architecture as the base model for all methods. For our method, we train the model based on a tiny validation set (a mini-batch). Specifically, before training, we randomly sample a mini-batch precisely labeled instances as the validation set. For the comparing methods, we add the validation examples into training set, which indicates that the classifier is trained on the original training examples and validation examples. Note that the comparing methods directly train the classifier with the validation examples, while MoGD only uses them for auxiliary disambiguation.
Furthermore, the parameters are determined as suggested in their original papers. For CIFAR-10 and CIFAR-100, we perform a strong augmentation (containing Cutout \cite{devries2017improved} and random horizontal flip) for each training image as done in \cite{sohn2020fixmatch} \footnote{https://github.com/kekmodel/FixMatch-pytorch}.
We employ the SGD optimizer \cite{robbins1951stochastic} with the momentum of 0.9 to train the model for 500 epochs. For all experiments, we repeat $5$ independent experiments (with different random seeds) and report the average results and standard deviation. All experiments are carried on Pytorch \footnote{https://pytorch.org/}\cite{paszke2019pytorch} with GeForce RTX 3080Ti GPUs.

\begin{table}[htbp]
	\centering
	\caption{Accuracy (Mean$\pm$Std) Comparison Results on Fashion-MNIST and CIFAR-10 with \textbf{Instance-Dependent} Partial Labels. We Use \textbf{ResNet} As the Backbone Network.}
	\label{tab:instance-dependent}
	\begin{tabular}{c|cc}
		\toprule
		Method & Fashion-MNIST & CIFAR-10 \\
		\midrule
		MoGD & $\mathbf{85.68 \pm 0.12\%}$ & $\mathbf{80.76 \pm 0.17\%}$ \\
		VALEN & $83.55 \pm 0.21\%$ & $66.37 \pm 0.18\%$ \\
		LWS   & $84.36 \pm 0.34\%$ & $48.82 \pm 0.25\%$ \\
		PRODEN& $83.12 \pm 0.46\%$ & $75.43 \pm 0.65\%$ \\
		RC    & $84.74 \pm 0.14\%$ & $74.14 \pm 0.16\%$ \\
		CC    & $83.88 \pm 0.10\%$ & $77.45 \pm 0.23\%$ \\
		\bottomrule
	\end{tabular}
\end{table}
\begin{table*}[!tb]
	\caption{Accuracy (Mean$\pm$Std) Comparison Results on Real-World Datasets with \textbf{Linear} Model. The Best Results are Shown in Bold.}
	\label{tab:result real-world}
	\centering
	\begin{tabular}{c| c c c c c}
		\toprule
		& Lost & MSRCv2 & BirdSong & Soccer Player & Yahoo! News\\
		\midrule
		MoGD & $74.65 \pm 1.43\%$& $45.38 \pm 1.89\%$ & $\mathbf{72.69 \pm 1.65\%}$ &  $\mathbf{57.86 \pm 0.32\%}$ & $\mathbf{68.82 \pm 0.44\%}$  \\
		\midrule

		VALEN & $ 72.34 \pm 2.26 \%$ & $47.36 \pm 1.82 \%$ & $ 72.22 \pm 0.36 \%$ & $55.28\pm 0.62\%$ & $67.48 \pm 0.26\%$\\
		LWS & $ \mathbf{76.54 \pm 2.17 \%}$ & $46.85 \pm 1.94\%$ & $ 72.36 \pm 1.48 \%$ & $56.34 \pm 2.14\%$ & $67.64 \pm 0.36\%$\\

		PRODEN          & $76.39\pm1.47\%$ & $45.27\pm1.73\%$ & $72.01\pm0.44\%$ & $55.99\pm0.58\%$ & $67.40\pm 0.55\%$\\
		$\text{D}^2$CNN & $69.61\pm5.48\%$ & $40.17\pm1.99\%$ & $66.58\pm1.49\%$ & $49.06\pm0.15\%$ & $56.39\pm0.89\%$ \\
		GA              & $48.21\pm4.44\%$ & $22.30\pm2.71\%$ & $29.77\pm1.43\%$ & $51.88\pm0.44\%$ & $34.32\pm0.95\%$ \\
		SURE            & $71.61\pm3.44\%$ & $31.57\pm2.48\%$ & $58.04\pm1.22\%$ & $49.16\pm0.20\%$ & $45.73\pm0.90\%$ \\
		CLPL            & $76.17\pm1.81\%$ & $43.64\pm0.24\%$ & $67.56\pm1.12\%$ & $49.88\pm4.29\%$ & $53.74\pm0.95\%$ \\
		ECOC            & $63.93\pm5.45\%$ & $46.78\pm2.84\%$ & $71.47\pm1.24\%$ & $55.51\pm0.54\%$ & $64.78\pm0.78\%$ \\
		PLSVM           & $72.86\pm5.45\%$ & $38.97\pm4.62\%$ & $60.46\pm1.99\%$ & $46.15\pm1.00\%$ & $60.46\pm1.48\%$ \\
		PL$k$NN         & $35.00\pm4.71\%$ & $41.60\pm2.30\%$ & $64.22\pm1.14\%$ & $49.18\pm0.26\%$ & $40.30\pm0.90\%$ \\
		IPAL            & $71.25\pm1.40\%$ & $\mathbf{52.36\pm2.87\%}$ & $71.19\pm1.54\%$ & $54.41\pm0.56\%$ & $66.22\pm0.80\%$ \\
		\bottomrule
	\end{tabular}
\end{table*}

\subsection{Experiments on Image Datasets}
We first report the results of the experiments on three image datasets with uniform candidate labels.

Table \ref{tab:result fashion} demonstrates the comparison results on Fashion-MNIST with linear models. We consider flipping probability $q \in \{0.1, 0.3, 0.5, 0.7\}$. As shown in the table, it can be observed that: 1) MoGD consistently outperforms all comparing methods. 2) Even compared to fully supervised learning, the results show that MoGD achieves competitive performance. In particular, with $q=0.1,0.3$, our method has a very small performance decrease (less than $\boldsymbol{1}\%$), even with $q = 0.5, 0.7$, the performance drop of MoGD is less than $\boldsymbol{1.5}\%, \boldsymbol{3}\%$, respectively, when compared to fully supervised learning.

Table \ref{tab:result cifar cnn} and Table \ref{tab:result cifar resnet} report the comparing results on CIFAR-10 and CIFAR-100 with ConvNet and Resnet, respectively. From the results, it can be observed that:
\begin{enumerate}
	\item MoGD achieves the best performance on all cases with various flipping probabilities.
	\item MoGD achieves comparable performance compared with fully supervised learning with low flipping probabilities on CIFAR-10 and CIFAR-100, e.g., with the ConvNet backbone, the performance drop of $\boldsymbol{1.07}\%$ on CIFAR-10 ($q=0.1$) and $\boldsymbol{1.42}\%$ on CIFAR-100 ($q=0.01$).
	\item MoGD outperforms the comparing methods with a large performance gap on CIFAR-100. For example, by using the ConvNet backbone, our method outperforms the second best method PRODEN by $\mathbf{4.63\%}$, $\mathbf{6.57\%}$, $\mathbf{10.44\%}$, and $\mathbf{12.48\%}$ with flipping probabilities $q=0.01, q=0.05, q=0.1, q=0.2$, respectively.
	\item MoGD achieves a relatively small performance drop as flipping probabilities increase. On the contrary, the comparing methods like LWS, VALEN, MSE and EXP work well on low flipping probabilities, while achieving a significant performance drop with the increase of flipping probabilities.
\end{enumerate}

These experimental results convincingly validate that MoGD can achieve favorable disambiguation ability for candidate labels. The results on CIFAR-100 demonstrate that MoGD show stronger robustness against the comparing methods.

Moreover, we analyze the convergence properties of these methods. Fig. \ref{fig:acc_cifar10} illustrates the performance curves of different methods with regard to test accuracy with the increase of epoch. From the figures, we can see that MoGD can quickly converge with about 40 epochs. Comparing methods like VALEN and LWS converge relatively slowly.

Besides the uniform candidate labels, we also compare these methods on Fashion-MNIST and CIFAR-10 with instance-dependent candidate labels.
Table \ref{tab:instance-dependent} shows that the proposed approach achieves the best performance in all cases. In particular, MoGD outperforms the second best method with $\boldsymbol{0.94}\%$ and $\boldsymbol{3.31}\%$ accuracy on Fashion-MNIST and CIFAR-10, respectively.

\subsection{Experiments on Real-World Datasets}
We also conduct experiments on five real-world PLL datasets to evaluate the practical usefulness of MoGD. Specifically, for all methods, we perform five-fold cross-validation. We use the same Linear model and employ the SGD optimizer \cite{robbins1951stochastic} with a momentum of $0.9$ to train the model for $500$ epochs. Table \ref{tab:result real-world} shows the comparison results between MoGD and comparing methods. From the results, it can observed that our method outperforms all comparing methods on BirdSong, Soccer Player and Yahoo! News. In particular, MoGD outperforms the second best method LWS $\boldsymbol{1.52}\%$ and $\boldsymbol{1.18}\%$ accuracy on Soccer Player and Yahoo! News, respectively. Although the performance of our method is lower than that of some methods on Lost and MSRCv2, it is still acceptable and higher than that of most comparing methods.

\subsection{Study On the Size of Validation Set}
In this section, we conduct experiments to study the influence of the size of validation data by comparing the performance of our method with different validation set size. Fig. \ref{fig:meta_size} illustrates the test accuracy curves of MoGD as the size of the validation set changes among $\{32, 64, 128, 256\}$ on the CIFAR-10, $\{128, 256, 384, 512\}$ on CIFAR-100 and $\{50, 100, 150, 200, 300, 400\}$ on the two real-world datasets include Soccer Player and Yahoo! News.
From the figures, the performance of MoGD is also satisfactory when the size of the validation set is relatively small, such as 128 for CIFAR-10, 256 for CIFAR-100, 100 for Yahoo! News, and 200 for Soccer Player. In fact, the meta objective on the validation set can encourage the classifier net to precisely estimate the confidence of the candidate label, which is similar to the effect of a regularization term.
This phenomenon suggests that, in practical applications, one can precisely annotate a few examples in advance to improve the classification performance.
\begin{figure*}[!tb]
	\centering
	\subfloat[CIFAR-10, ConvNet, $q=0.1$]
	{
		\includegraphics[width=0.23\textwidth]{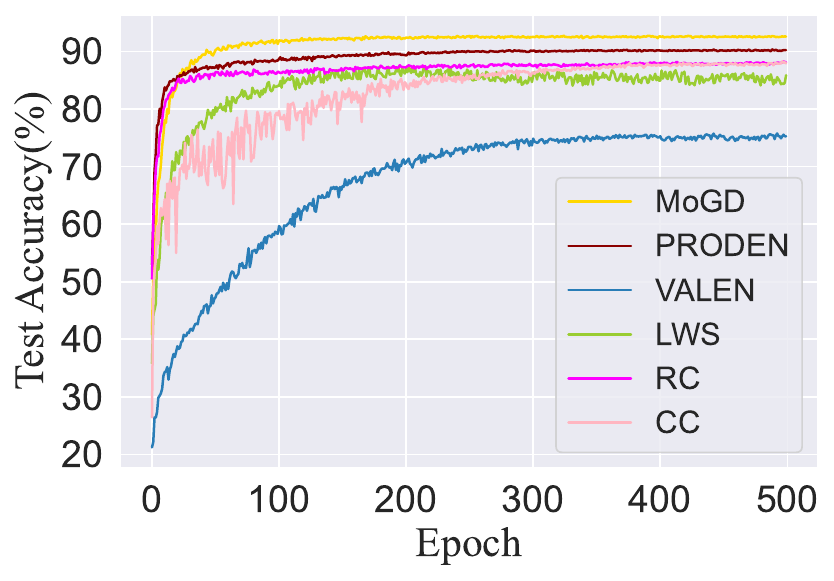}
	}
	\hfil
	\subfloat[CIFAR-10, ConvNet, $q=0.5$]
	{
		\includegraphics[width=0.23\textwidth]{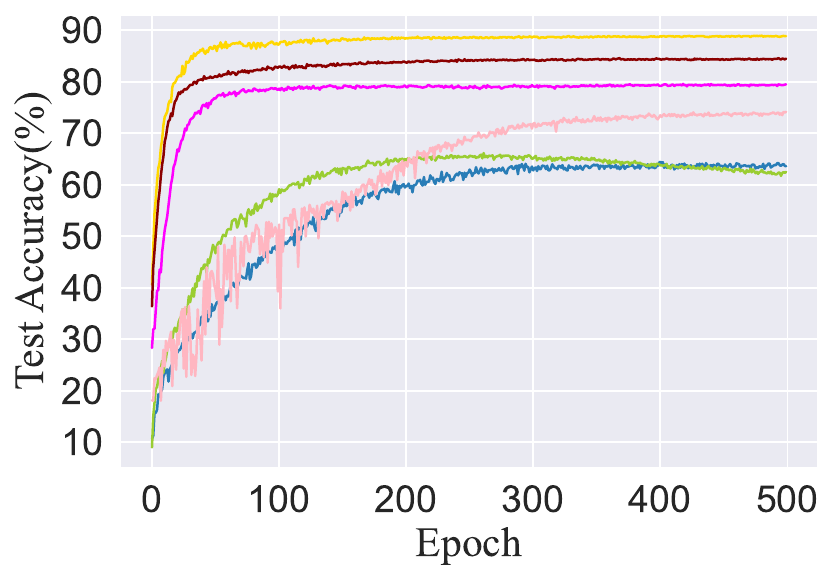}
	}
	\hfil
	\subfloat[CIFAR-10, ResNet, $q=0.1$]
	{
		\includegraphics[width=0.23\textwidth]{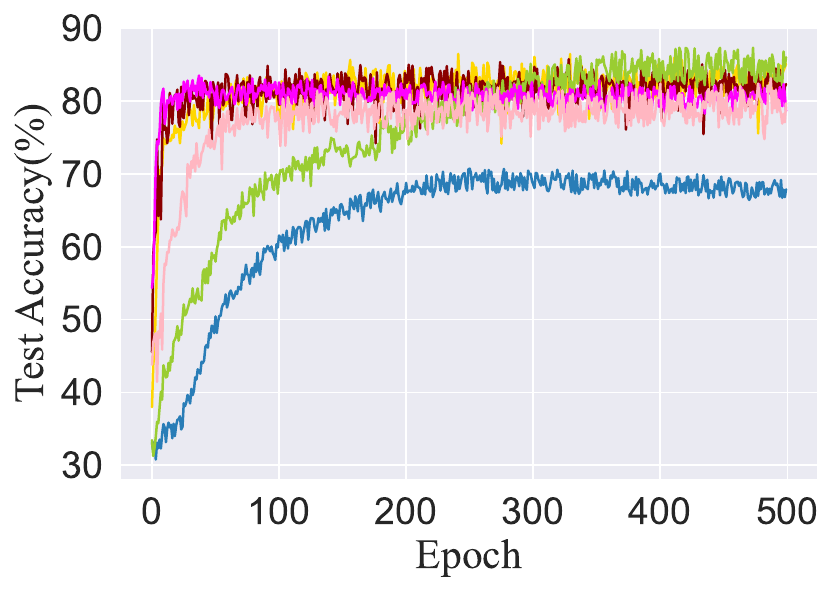}
	}
	\hfil
	\subfloat[CIFAR-10, ResNet, $q=0.5$]
	{
		\includegraphics[width=0.23\textwidth]{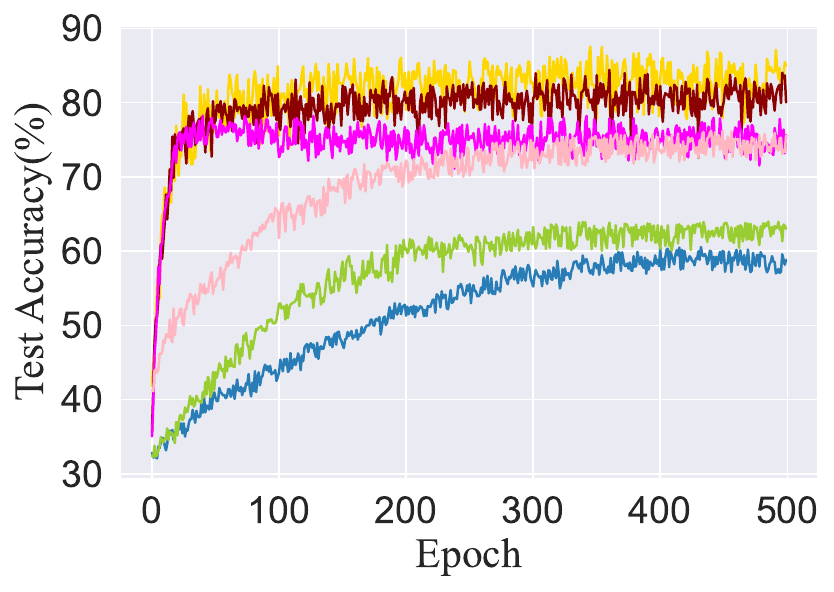}
	}

	\caption{Test accuracy curve for various models and datasets. Different colors represent the corresponding methods.}
	\label{fig:acc_cifar10}
\end{figure*}

\begin{figure*}[!tb]
	\centering
	\subfloat[CIFAR-10]
	{
		\includegraphics[width=0.23\textwidth]{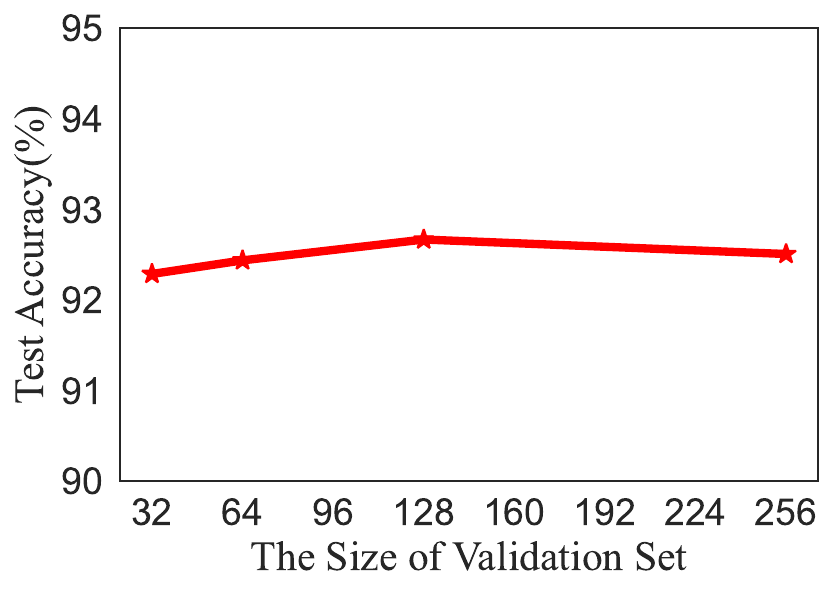}
		\label{fig:1}
	}
	\hfil
	\subfloat[CIFAR-100]
	{
		\includegraphics[width=0.23\textwidth]{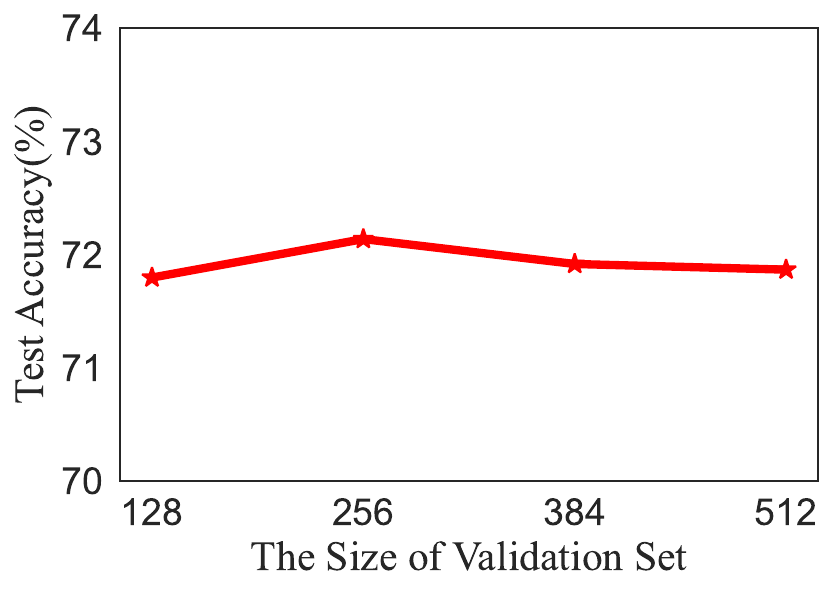}
		\label{fig:2}
	}
	\hfil
	\subfloat[Yahoo! News]
	{
		\includegraphics[width=0.23\textwidth]{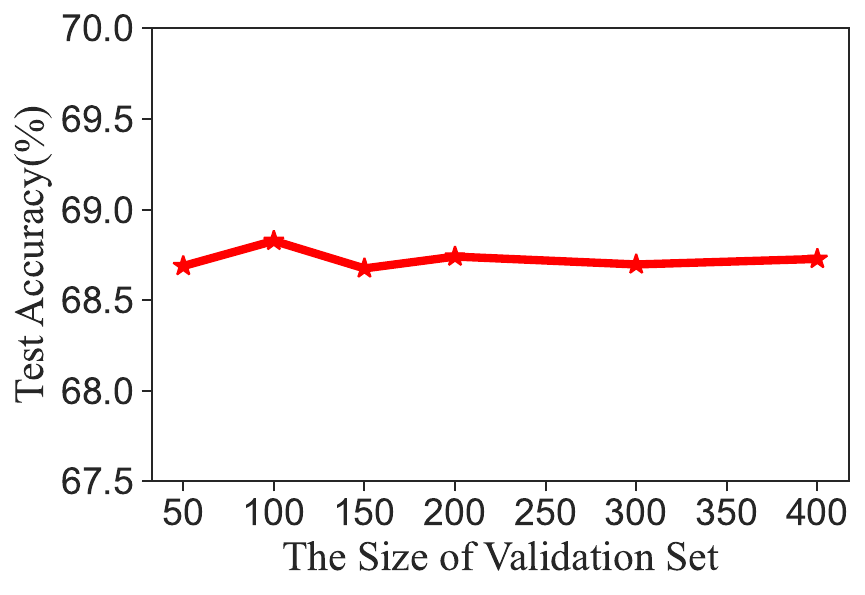}
	}
	\hfil
	\subfloat[Soccer Player]
	{
		\includegraphics[width=0.23\textwidth]{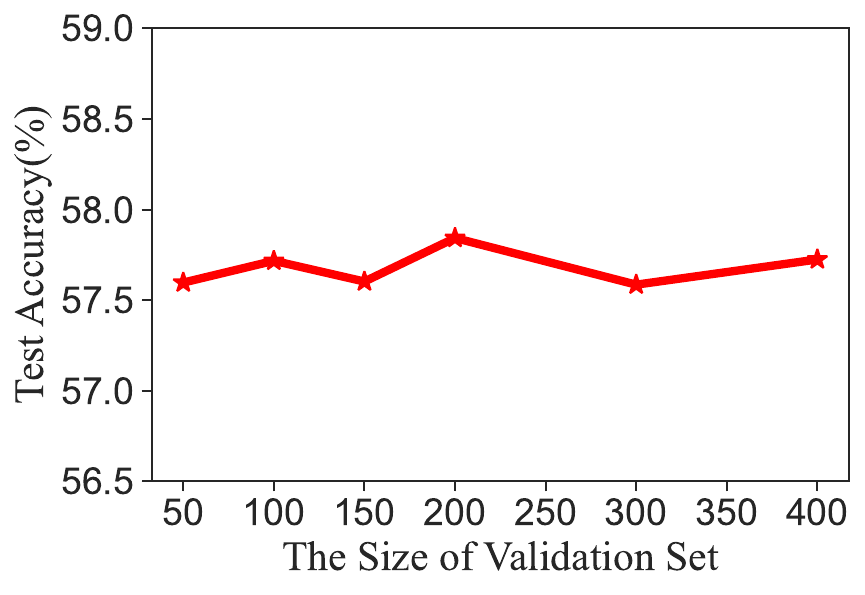}
	}

	\caption{Test accuracy curve of MoGD on CIFAR-10, CIFAR-100, Yahoo! News and Soccer Player with varying size of validation set. For CIFAR-10, the ambiguity level $q=0.1$, and for CIFAR-100, the ambiguity level $q=0.01$.}
	\label{fig:meta_size}
\end{figure*}
\section{Conclusion}
\label{sec:conclusion}
In this paper, we propose a novel framework for partial label learning by achieving disambiguation for candidate labels in a meta-learning manner. Different from the previous methods, we utilize the meta-objective on a tiny clean validation set to adaptively estimate the confidence of each candidate label without extra assumptions of data. Then we optimize the multi-class classifier by minimizing a confidence-weighted objective function. To improve the training efficiency, we iteratively update these two objective functions by using an online approximation strategy. We validate the effectiveness of MoGD both theoretically and experimentally. In theory, we prove that the model trained by MoGD is never worse than merely learning from partial-labeled data. Extensive experiments on commonly used benchmark datasets and real-world datasets demonstrate that MoGD is superior to the state-of-the-art methods. In the future, we will focus on studying other more powerful learning models to further enhance the performance of MoGD algorithm.
\\

\bibliographystyle{IEEEtran}
\bibliography{reference}

\vfill

\end{document}